\documentclass[10pt,twocolumn,letterpaper]{article}

\usepackage[pagenumbers]{cvpr} 

\usepackage{amsmath,amsfonts,bm}

\def\eqref#1{equation~\ref{#1}}

\def\1{\bm{1}}

\DeclareMathAlphabet{\mathsfit}{\encodingdefault}{\sfdefault}{m}{sl}
\SetMathAlphabet{\mathsfit}{bold}{\encodingdefault}{\sfdefault}{bx}{n}

\DeclareMathOperator*{\argmin}{arg\,min}

\usepackage{url}

\usepackage{algorithm, algpseudocode}
\usepackage{amsmath}
\usepackage{amssymb}
\usepackage{mathtools}
\usepackage{mathrsfs}
\usepackage{color}
\usepackage{xcolor}
\usepackage{bm}
\usepackage{appendix}
\usepackage{listings}
\usepackage{amsthm}
\usepackage{multirow}
\usepackage{multicol}

\usepackage[utf8]{inputenc} 
\usepackage[T1]{fontenc}    
\usepackage{amsfonts}       
\usepackage{nicefrac}       
\usepackage{microtype}      
\usepackage{xcolor}         

\usepackage{graphicx}
\usepackage{amsfonts}
\usepackage{bm,amsmath}
\usepackage{amsthm}
\usepackage{xspace}
\usepackage{booktabs} 
\usepackage{threeparttable}
\usepackage{array}
\usepackage{enumitem}

\newtheorem{theorem}{Theorem}

\newtheorem{lemma}{Lemma}

\newtheorem{proposition}{Proposition}
\numberwithin{equation}{section}

\definecolor{UofT}{rgb}{0, 0.5, 0.639}
\definecolor{NJU}{rgb}{0.415, 0.0, 0.372}
\definecolor{HKU}{rgb}{0.0, 0.6, 0.0}
\newcommand{\s}[1]{\mathbf{#1}}
\newcommand{\m}[1]{\bm{#1}}

\renewcommand{\eqref}[1]{(\ref{#1})}

\newcommand{\shorthyphen}{{\scalebox{0.5}[1.0]{$-$}}}

\def\l{\mathcal{L}}
\def\n{\mathcal{N}}

\def\argmin{\mathop{\arg\min}}
\def\prox{\mathrm{prox}}

\newcommand{\OurMethod}{\emph{ADMMDiff}}

\definecolor{cvprblue}{rgb}{0.21,0.49,0.74}
\usepackage[pagebackref,breaklinks,colorlinks,allcolors=cvprblue]{hyperref}
\usepackage{colortbl}

\def\paperID{14464} 
\def\confName{CVPR}
\def\confYear{2025}
\def\l{\mathcal{L}}

\title{Decoupling Training-Free Guided Diffusion by ADMM}

\author{
    Youyuan Zhang\textsuperscript{\rm 1}\thanks{Authors contributed equally to this work.} \quad
    Zehua Liu\textsuperscript{\rm 2}\footnotemark[1] \quad
    Zenan Li\textsuperscript{\rm 3}\footnotemark[1] \quad
    Zhaoyu Li\textsuperscript{\rm 4}\footnotemark[1] \quad
    James J. Clark\textsuperscript{\rm 1} \quad
    Xujie Si\textsuperscript{\rm 4,5}
    \\
    \textsuperscript{\rm 1}McGill University \quad
    \textsuperscript{\rm 2}The University of Hong Kong \quad
    \textsuperscript{\rm 3}Nanjing University
    \\
    \textsuperscript{\rm 4}University of Toronto \quad
    \textsuperscript{\rm 5}CIFAR AI Chair, MILA
    \\
    {\tt\small youyuan.zhang@mail.mcgill.ca, liuzehua@connect.hku.hk, lizn@smail.nju.edu.cn}
    \\
    {\tt\small \{zhaoyu,six\}@cs.toronto.edu, jame.clark1@mcgill.ca}
}

\begin{document}




\makeatletter
\let\@oldmaketitle\@maketitle
    \renewcommand{\@maketitle}{\@oldmaketitle
    \centering
    \includegraphics[width=0.98\textwidth]{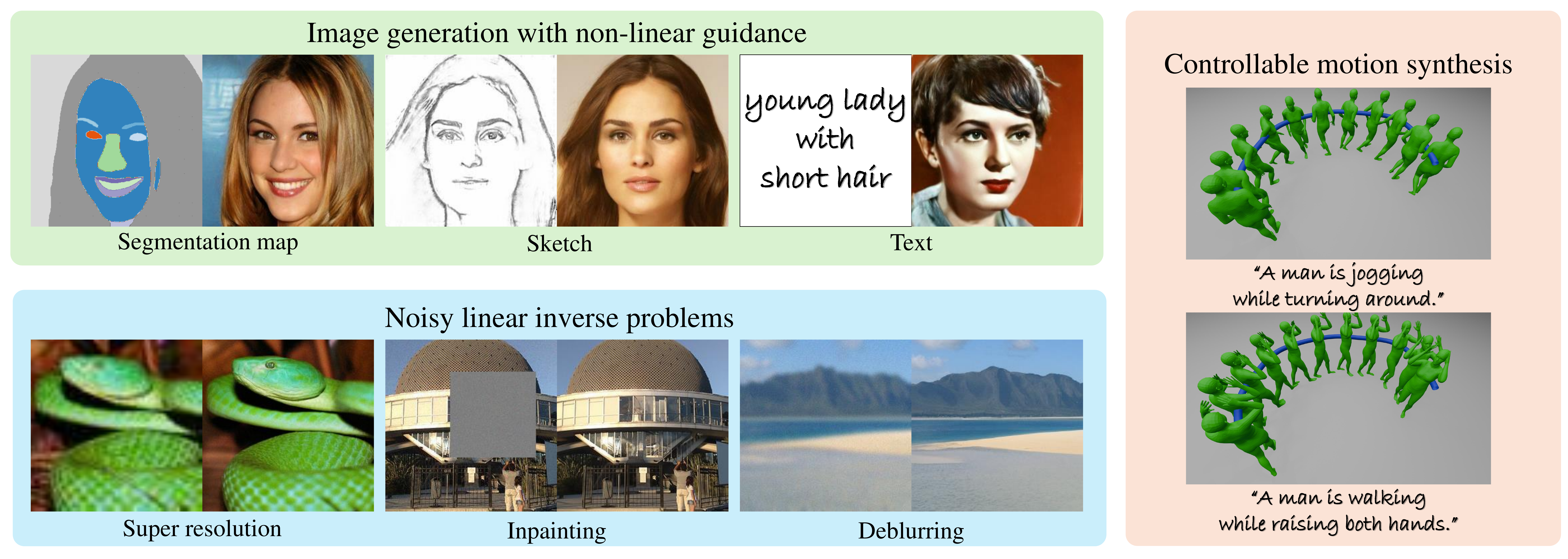}
    \captionof{figure}{\textbf{Illustrated results of \OurMethod{} on diverse conditional generation tasks.} \OurMethod{} effectively guides the generation process of diffusion models using training-free guidance functions, producing high-quality samples that adhere closely to the specified conditions.}
    \label{fig:teaser}
  \bigskip}
\makeatother

\maketitle

\begin{abstract}
In this paper, we consider the conditional generation problem by guiding off-the-shelf unconditional diffusion models with differentiable loss functions in a plug-and-play fashion. While previous research has primarily focused on balancing the unconditional diffusion model and the guided loss through a tuned weight hyperparameter, we propose a novel framework that distinctly decouples these two components. Specifically, we introduce two variables $\s{x}$ and $\s{z}$, to represent the generated samples governed by the unconditional generation model and the guidance function, respectively. This decoupling reformulates conditional generation into two manageable subproblems, unified by the constraint $\s{x} = \s{z}$. Leveraging this setup, we develop a new algorithm based on the Alternating Direction Method of Multipliers (ADMM) to adaptively balance these components. Additionally, we establish the equivalence between the diffusion reverse step and the proximal operator of ADMM and provide a detailed convergence analysis of our algorithm under certain mild assumptions. Our experiments demonstrate that our proposed method \OurMethod{} consistently generates high-quality samples while ensuring strong adherence to the conditioning criteria. It outperforms existing methods across a range of conditional generation tasks, including image generation with various guidance and controllable motion synthesis.
\end{abstract}

\section{Introduction}
\label{sec:introduction}
Diffusion models~\citep{sohl_2015_deep, ho_2020_denoising, song_2020_score} have emerged as a potent paradigm in generative modeling, demonstrating versatility across various domains, including image synthesis~\citep{ rombach_2022_high,saharia_2022_photorealistic}, 3D object generation~\citep{poole_2022_dreamfusion,lin_2023_magic3d}, natural language processing~\citep{li_2022_diffusion, lovelace_2022_latent}, and time-series analysis~\citep{rasul2021autoregressive, bilovs_2022_modeling}. In many of these areas, there is an escalating interest in conditional generation, where the generative process is guided not only by the diffusion model but also by external information such as text prompts~\citep{kim_2022_diffusionclip, avrahami_2022_blended, hertz_2022_prompt, ramesh2022hierarchical, kawar_2023_imagic}, segmentation maps~\citep{zhang_2023_adding, mou2024t2i}, sketches~\citep{voynov2023sketch, zhang_2023_adding, mou2024t2i}, etc. In this paper, we focus on solving this task by combining an unconditional diffusion model with a differentiable guidance function that evaluates condition satisfaction in a \textit{plug-and-play} fashion~\citep{graikos2022diffusion, chung2022diffusion, song2023loss}, which generates the desired samples without additional training.

Specifically, our objective is to generate samples from the prior distribution $p(\s{x})$ that satisfy conditions $\s{y}$, approximating the posterior $p(\s{x}|\s{y})$. Using Bayes' rule, the posterior can be expressed as $p(\s{x}|\s{y}) \propto p(\s{x}) p(\s{y}|\s{x})$, where $p(\s{x})$ is modeled by a pretrained diffusion model and $p(\s{y}|\s{x})$ is approximated by a differentiable loss function. This formulation naturally divides the conditional generation task into two distinct components. However, effectively integrating these components poses significant challenges due to their different objectives. For example, the model may generate diverse but condition-violating samples or produce samples that satisfy the conditions but lack diversity or quality. Existing approaches~\citep{graikos2022diffusion, chung2022diffusion, song2023loss, he2023manifold, yu2023freedom, ye2024tfg} mainly introduce a weight hyperparameter to strike a balance between the unconditional diffusion model and the guidance function. Nonetheless, determining an optimal 
parameter is non-trivial and highly dependent on the specific task, making it difficult to generalize across different problem settings.

To address these limitations, we propose a novel framework that fundamentally rethinks the interaction between the unconditional diffusion model and the guidance function in conditional generation. Specifically, we decouple these two components by representing the samples from the diffusion model as $\s{x}$ and introducing an auxiliary variable $\s{z}$ that represents samples refined by the guidance function. This formulation allows us to transform the conditional generation problem into two more tractable subproblems, connected by the constraint $\s{x} = \s{z}$. Within this setup, we introduce a dual variable and develop a new algorithm based on the Alternating Direction Method of Multipliers (ADMM)~\citep{glowinski_1975_approximation, gabay_1976_dual} for conditional generation. Our algorithm sidesteps the use of the weight hyperparameter and allows for a more natural and adaptive balancing between the diffusion model and the guidance function. Additionally, we theoretically demonstrate that the proximal operator of the diffusion term $-\log p(\s{x})$ in the ADMM can be effectively approximated by the diffusion reverse process and further provide a rigorous proof of the convergence of our algorithm under generic assumptions regarding the distributions $p(\s{x})$ and $p(\s{y}|\s{z})$. 

We evaluate our proposed framework on various conditional generation tasks in different domains (see Figure~\ref{fig:teaser} for illustration). In image synthesis tasks, our method can accept both non-linear semantic parsing models and linear measurement functions as guidance conditions, and achieves superior performance in terms of both image quality and condition satisfaction. In addition, our method is capable of generalizing in motion domain and guide text-condition motion diffusion models to follow specific trajectories.
\section{Preliminary}
\label{sec:preliminary}

\subsection{Diffusion Model}

Given any datapoint $\s{x}_0 \sim p(\s{x})$, the \emph{forward diffusion process} adds small amount of Gaussian noise to the sample by $T$ steps~\citep{ho_2020_denoising}, producing a sequence of noisy samples $\s{x}_t, t=1,\dots,T$: 
\begin{equation*}
    \s{x}_t = \sqrt{1 - \beta_t} \s{x}_{t-1} + \sqrt{\beta_t} \m{\epsilon}, \quad \m{\epsilon} \sim \n (\s{0}, \s{I}),
\end{equation*}
where $\{ \beta_t \in (0, 1) \}_{t=1}^T$ is a pre-defined variance schedule to control the scales of Gaussian noise.
By introducing auxiliary variables $\alpha_t := 1 - \beta_t$ and $\bar{\alpha}_t := \prod_{i=1}^{t} \alpha_i$ for $t=1, 2, \dots, T$, we have 
\begin{equation*} \label{equ: 1.2}
    \s{x}_t = \sqrt{\bar{\alpha}_t} \s{x}_0 + \sqrt{1 - \bar{\alpha}_t} \m{\epsilon}, \quad \m{\epsilon} \sim \n (\s{0}, \s{I}).
\end{equation*}
For the distribution $p_t(\s{x}_t)$ induced by adding noise to $p(\s{x})$, the score function $\nabla_{\s{x}_t} \log p_t(\s{x}_t)$ can be approximated by training a time-conditioned neural network $s_{\bm{\theta}}(\s{x}_t, t) \colon \mathcal{X} \times [0, T] \to \mathcal{X}$ that tries to denoise the noisy sample $\s{x}_t$ using the denoising score matching objective~\citep{song_2020_score}:
\begin{equation*}
L_t(\theta) = \mathbb{E}_{\s{x}_0, \s{x}_t}\left[\left\|s_{\bm{\theta}}(\s{x}_t, t)-\s{x}_0\right\|_2^2\right],
\end{equation*}
With sufficient capacity of a denoising model $s_\theta$, the \emph{reverse diffusion process} is obtained using Tweedie’s 
formula~\citep{efron2011tweedie}:
\begin{equation*} \label{eqn:est_mu}
\tilde{\m{\mu}}_{\m{\theta}}(\s{x}_t, t) := \frac{1}{\sqrt{\alpha_t}} (\s{x}_t +  \beta_t s_\theta (\s{x}_t, t))
\end{equation*}
To approximately obtain samples, the Langevin dynamic is often introduced~\citep{bussi2007accurate, song2019generative}, forming the following update rule:
\begin{equation*} \label{eqn:est_x}
\tilde{\s{x}}_{t-1} = \frac{1}{\sqrt{\alpha_t}} \left( \s{x}_t +  \beta_t s_\theta (\s{x}_t, t) \right) + \sqrt{\frac{1-\bar{\alpha}_{t-1}}{1-\bar{\alpha}_t} \beta_t} \m{\epsilon}.
\end{equation*} 

\subsection{Proximal Operator and ADMM}

\begin{figure*}[!tbp]
     \centering
     \begin{subfigure}[b]{0.49\textwidth}
         \centering
         \includegraphics[width=\textwidth]{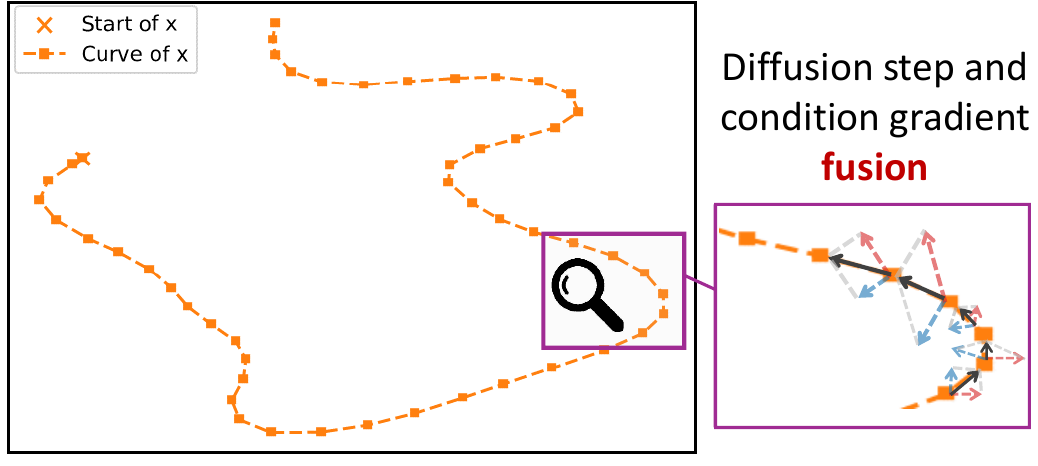}
         \caption{Geometry of single trajectory denoising methods.}
         \label{fig:geo_dps}
     \end{subfigure}
     \begin{subfigure}[b]{0.49\textwidth}
         \centering
         \includegraphics[width=\textwidth]{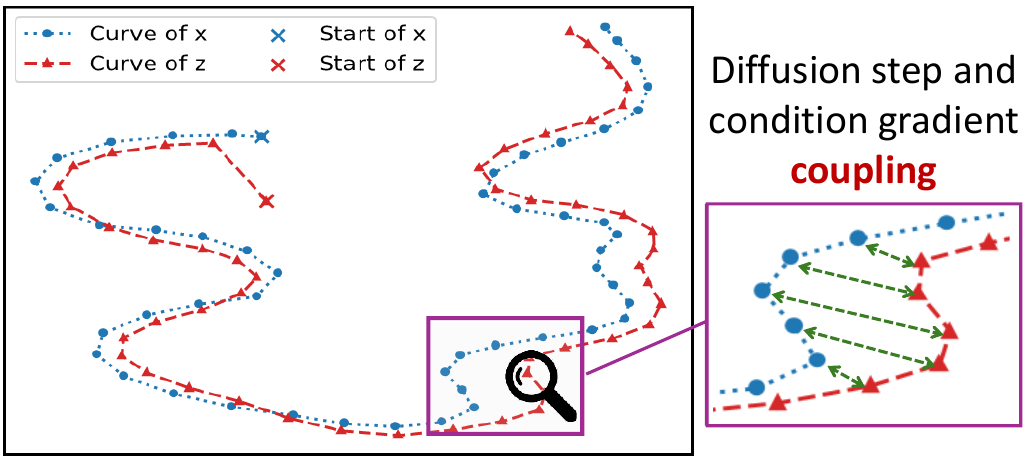}
         \caption{Geometry of ADMM-based method for decoupling guidance gradient.}
         \label{fig:geo_admm}
     \end{subfigure}
    \caption{\textbf{Geometrical illustration of ADMM-based method.} Compared with classic guided diffusion frameworks which directly perturbs reverse diffusion steps with guidance gradients, ADMM-based method decouples the guidance gradient from the reverse diffusion trajectory and allows more flexibility to explore guidance conditions.}
    \label{fig:geometry}
\end{figure*}

For any $l$-weakly convex function $g$ (i.e., the function satisfying that $g(\s{x}) + (l/2)\|\s{x}\|^2$ is convex), its proximal operator with $\lambda \in (0, 1/l)$ is defined by~\citep{moreau_1965_proximite, parikh2014proximal}:
    \begin{equation*}
        \prox_{\lambda g} (\s{x}) := \argmin_{\s{x}'} \big\{g (\s{x}') + \frac{1}{2 \lambda} \| \s{x}' - \s{x} \|^2 \big\}.
    \end{equation*}
Note that the mapping $\s{x}' \to g(\s{x}') + (1/ {2 \lambda}) \| \s{x}' - \s{x} \|^2$ is strongly convex for any $\s{x}'$, and thus the proximal operator is $\frac{1}{1-\lambda\ell}$-Lipschitz~\citep[Proposition 12.19]{rockafellar2009variational}.

Alternating Direction Method of Multipliers (ADMM) is a classic method~\citep{glowinski_1975_approximation,gabay_1976_dual}, which efficiently addresses the following constraint minimization problem.
\begin{equation*} 
\begin{aligned}
    \min_{\s{x}, \s{z}} & \quad f(\s{x}) + g(\s{z}), \\ 
    \mathrm{s.t.} & \quad \s{x} = \s{z}.
\end{aligned}
\end{equation*}
Furthermore, the corresponding augmented Lagrangian function is defined as
\begin{equation*}
     \mathcal{L}_{\rho} (\s{x}, \s{z}; \m{\mu}) := f(\s{x}) + g(\s{z}) + \langle \m{\nu}, \s{x} - \s{z} \rangle + \frac{\rho}{2} \| \s{x} - \s{z} \|^2.
\end{equation*}
Based on the augmented Lagrangian function and plugging in the proximal operator, the Gauss-Seidel update of ADMM at the $t$-th iteration can be formulated by
\begin{equation*}
\begin{dcases}
& \s{x}_{t+1} = \prox_{\rho f} \big(\s{z}_t - \frac{1}{\rho} \m{\nu}_t \big); \\
& \s{z}_{t+1} = \prox_{\rho g} \big(\s{x}_{t+1} + \frac{1}{\rho} \m{\nu}_t \big); \\ 
& \m{\nu}_{t+1} = \m{\nu}_k + \rho (\s{x}_{t+1} - \s{z}_{t+1}).
\end{dcases}
\end{equation*}
ADMM not only guarantees the convergence, but also achieves a satisfactory convergent rate when subproblems involving the proximal operator can be solved efficiently~\citep{liu_1994_numerical, jiang_2014_alternating, magnusson_2015_convergence}. 
Moreover, as a primal-dual method, ADMM possesses a nice property that can dynamically control the equilibrium between the objective function and the constraint satisfaction in the optimization process, by balancing the primal and dual residuals~\citep{boyd_2011_distributed, esser2010general}.

\section{Conditional Generation by ADMM}
\label{sec:framework}

We frame the conditional generation problem as the maximum likelihood problem:
\begin{equation*} 
    \max_{\s{x}} \, \log p(\s{x} | \s{y}) = \log p(\s{x}) + \log p(\s{y} | \s{x}) - \log p (\s{y}),
\end{equation*}
where $\log p (\mathbf{y})$ is constant for a specific condition $\mathbf{y}$, $p (\mathbf{x})$ and $ p(\mathbf{y} | \mathbf{x})$ are approximated by an off-the-shelf diffusion model (denoted by $q_{\m{\phi}}(\s{x})$) and a pre-defined condition guidance function (denoted by $c_{\bm{\theta}}(\s{x}, \s{y})$), respectively.
Putting these together, the conditional generation problem
can be rewritten as the following variational version:
\begin{equation*} \label{equ:var_problem}
    \max_{\s{x}}  \, \log q_{\m{\phi}}(\s{x}) + \log c_{\m{\theta}}(\s{x}, \s{y}).
\end{equation*}
Furthermore, by \emph{decoupling} the two objectives with an auxiliary variable $\s{z}$, 
we can reformulate the problem
as
\begin{equation} \label{equ:primal-dual} \tag{P}
    \max_{\s{x}, \s{z}}~ \log q_{\m{\phi}} (\s{x}) + \log c_{\m{\theta}}(\s{z}, \s{y}), \quad \mathrm{s.t.} \quad \s{x} = \s{z}.
\end{equation}
From this formulation, the sample generating and condition satisfying are assigned to two different variables $\s{x}$ and $\s{z}$, and they are connected by the constraint $\s{x} = \s{z}$. 
In this sense, the conditional generation is converted into a \emph{coupling} problem, i.e., fusing $\s{x}$ and $\s{z}$ in the generative process~\citep{den2012probability, lindvall2002lectures}, such that $\s{x}$ and $\s{z}$ maximizes $\log q_{\m{\phi}} (\s{x})$ and $\log c_{\m{\theta}}(\s{z}, \s{y})$ under the equality constraint $\s{x} = \s{z}$.

Rewriting the maximization problem to the minimization, the corresponding augmented Lagrangian function is
\begin{equation*}
\begin{aligned}
    \mathcal{L}_{\rho} (\s{x}, \s{z}; \m{\nu}) & 
    := - \log q_{\m{\phi}} (\s{x}) - \log c_{\m{\theta}} (\s{z}, \s{y}) \\
    & \qquad + \langle \m{\nu}, \s{x} - \s{z} \rangle + \frac{\rho}{2} \| \s{x} - \s{z} \|^2.
\end{aligned}
\end{equation*} 
Based on the augmented Lagrangian form, we encode the constraint as a penalty term, paired by the Lagrange multiplier (i.e., the dual variable $\m{\nu}$), as well as a coefficient $\rho > 0$, to control the intensity of the constraint enforcement.
Next, by applying the ADMM update, we have\footnote{To keep the consistency with the index order of reverse diffusion process, the iteration index $t$ is also defined to range from $T$ to $0$.}
\begin{equation*}
\begin{dcases}
& \s{x}_{t} = \prox_{(- \rho \log q_{\m{\phi}})} (\s{z}_{t+1} - \frac{1}{\rho} \m{\nu}_t ); \\
& \s{z}_{t} = \prox_{(- \rho \log c_{\m{\theta}})} (\s{x}_{t} + \frac{1}{\rho} \m{\nu}_t); \\ 
& \m{\nu}_{t} = \m{\nu}_{t+1} + \rho (\s{x}_{t} - \s{z}_{t}); 
\end{dcases}
\end{equation*}
In particular, the dual variable $\m{\nu}$ and the coefficient $\rho$ are \emph{dynamically} adapted (by conducting gradient descent in ADMM) based on the satisfaction degree of constraint $\s{x} = \s{z}$, which essentially controls the progress of coupling. 

Compared to existing training-free guided diffusion approaches, ADMM provides a more flexible paradigm by decoupling the guidance gradient from the reverse diffusion trajectory. Figure~\ref{fig:geometry} demonstrates an intuitive illustration of existing methods and our approach. Existing approaches directly modify reverse diffusion steps with guidance gradients, while the ADMM-based method gradually couples two trajectories of reverse diffusion and condition satisfaction.

However, the computations of the proximal operators in ADMM are still intractable due to the non-convex and non-analytic functions $q_{\m{\phi}}(\s{x})$ and $c_{\m{\theta}}(\s{z}, \s{y})$. 
To overcome this issue, we propose using the inexact version instead, in which the proximal operator of $\s{x}$ can be approximated by a single reverse step of the diffusion model, and the proximal operator of $\s{z}$ can be approximated by a few steps of gradient descent with respect to the condition guidance function. 
To achieve an approximation of the proximal operator of $-\log q_{\m{\phi}}$, we have the following proposition.

\begin{proposition} \label{prop:1}
Let $\s{x}_t = \sqrt{\bar{\alpha}} \s{x}_0 + \sqrt{1-\bar{\alpha}} \m{\epsilon}, \m{\epsilon} \sim \n (0, \s{I}),$ be a noisy point generated from the datapoint $\s{x}_0 \sim p(\s{x})$ by the diffusion forward process. Then, 
the point $\tilde{\s{x}}_{t-1}$ derived from the diffusion reverse step at ${\s{x}}_t$, i.e.,
\begin{equation*}
\tilde{\s{x}}_{t-1} = \frac{1}{\sqrt{\alpha_t}} \left( \s{x}_t +  \beta_t s_{\m{\theta}} (\s{x}_t, t) \right) + \sqrt{\frac{1-\bar{\alpha}_{t-1}}{1-\bar{\alpha}_t} \beta_t} \m{\epsilon}, 
\end{equation*} 
is a first-order approximation to the proximal operator of negative log-likelihood $\shorthyphen \frac{1}{\rho} \log q_{\m{\phi}}(\s{x})$ at ${\s{x}}_t$ with $\rho = \frac{\beta}{1-\beta}$. 
\end{proposition}

\emph{Remarks.} The proof of this proposition aligns with the theoretical analysis in~\citep{zhu2023denoising}, which provides a strong foundation for interpreting reverse diffusion as a proximal operator.

According to Proposition \ref{prop:1} illustrating the connection between the proximal operator and diffusion reverse process, we can directly approximate the update of $\s{x}$ using the off-the-shelf diffusion model. 
However, the update of $\s{z}$ is still challenging,
because the condition guidance function is generally implemented by an off-the-shelf classifier, which may not be explicitly trained on samples along the forward diffusion trajectory but rather on the original datapoint $\s{x}_0$.
Therefore, we first use Tweedie's formula as an estimation of $\s{z}_0$: 
\begin{equation*}
\tilde{\s{z}}_0 (\s{z}_t) = \frac{1}{\sqrt{\bar{\alpha}_t}} \left( \s{z}_t + (1 - \bar{\alpha}_t) \m{s}_t \right),
\end{equation*}
Then, we compute the loss function $c_{\bm{\theta}}(\tilde{\s{z}}_0, \s{y})$ and apply the gradient descent with respect to $\s{z}_t$ as the first-order approximation of the proximal operator:
\begin{equation*}
\begin{aligned}
\s{z}_{t-1} & = \s{z}_{t} - \eta \rho (\s{z}_t - \s{x}_t - \frac{1}{\rho} \m{\nu}_t) + \eta \nabla_{\s{z}_t} \log c_{\m{\theta}}(\tilde{\s{z}}_0(\s{z}_t), \s{y}).
\end{aligned}
\end{equation*}

\begin{table*}[t]
    \centering
    \resizebox{0.7\textwidth}{!}{
    \begin{tabular}{c | cc | cc | cc}
        \toprule
          \multicolumn{1}{c}{} \vline & \multicolumn{2}{c}{Segmentation map} \vline & \multicolumn{2}{c}{Sketch}\vline & \multicolumn{2}{c}{Text}\\
          \multicolumn{1}{c}{\textbf{Method}} \vline & Distance ↓ & FID ↓ & Distance ↓ & FID ↓ & Distance ↓ & FID ↓ \\
        \hline
        \rule{0pt}{10pt}
        DPS~\citep{chung2022diffusion} & 2199.8 & 57.38 & 50.74 & 67.21 & 10.46 & 57.13 \\\rule{0pt}{10pt}
        LGD-MC~\citep{song2023loss} & 2073.1 & 46.10 & 34.33 & 65.99 & 10.72 & 44.04 \\\rule{0pt}{10pt}
        FreeDoM~\citep{yu2023freedom} & 1696.1 & 53.08 & 33.29 & 70.97 & 10.83 & 55.91 \\\rule{0pt}{10pt}
        MPGD~\citep{he2023manifold} & 1922.5 & 43.97 & 35.32 & 60.56 & 10.70 & 43.98 \\\rule{0pt}{10pt}
        \textbf{ADMMDiff} & \textbf{1586.2} & \textbf{30.18} & \textbf{32.28} & \textbf{42.43} & \textbf{10.08} & \textbf{43.84} \\
        \bottomrule
    \end{tabular}}
    \caption{Comparison of non-linear guided image synthesis on CelebA-HQ dataset with different guidance conditions. \textbf{Bold} indicates the best. We use the $\ell_2$ distance between parsing models' outputs from generated samples and reference images when testing on segmentation and sketch guidance, and use the $\ell_2$ distance between text and image embeddings when testing on text guidance. We compute the FID score from the statistics of the all the generated images and reference image set. Results illustrate that our method outperforms all the baselines in term of both image quality and guidance satisfaction. }
    \label{table:face_dist_fid}
    \vspace{-1em}
\end{table*}

\begin{table*}[t]  
\vspace{1.em}
    \centering  
    \resizebox{0.7\textwidth}{!}{
    \begin{tabular}{c | c | c | c | c | c }  
        \toprule  
        Methods  & DPS~\citep{chung2022diffusion} & LGD-MC~\citep{song2023loss} & FreeDoM~\citep{yu2023freedom} & MPGD~\citep{he2023manifold} & \textbf{ADMMDiff} \\  
        \midrule  
        CLIP Score↑ & 24.5 & 24.3 & 25.9 & 25.1 & \textbf{26.85} \\
        \bottomrule  
    \end{tabular} } 
    \caption{Comparison of textual alignment on unconditional CelebA-HQ using text guidance. \textbf{Bold} indicates the best. The CLIP score is computed using CLIP-L/14.  Results illustrate that our method achieves the best performance.}  
    \label{table:face_clip}
    \vspace{-2em}
\end{table*}

\section{Algorithm and Theoretical analysis}
\label{sec:algorithm}
By integrating the update steps for both $\s{x}$ and $\s{z}$, we derive the complete ADMM-based conditional generation algorithm, as outlined in Algorithm~\ref{alg:admm}. 
To further minimize the approximation error of the proximal operator, we extend the gradient descent process to $K_t$ iterations at each step, ensuring more accurate optimization and improved convergence.

\begin{algorithm}[h]
\caption{Training-Free Guided Diffusion by ADMM}
\label{alg:admm} 

\quad {\bf Input:} pretrained diffusion model $s_{\theta}(\s{x}_t, t)$ and guidance function $c_{\phi}(\s{x}, \s{y})$; initial points $\s{x}_T, \s{z}_T$ sampled from $\n(\s{0}, \s{I})$, $\m{\nu}_T = \m{0}$, step size $\eta > 0$.  

\begin{algorithmic} \setlength{\baselineskip}{1.25\baselineskip}
\For{$t = T, T-1, \dots, 0,$}
\State \underline{\emph{Inexact update of $\s{x}$ by reverse diffusion}}
\State $\hat{\s{x}}_{t-1} = \s{z}_{t} - \frac{1}{\rho} \m{\nu}_t$; $\m{s}_{t} = s_\theta (\hat{\s{x}}_t, t)$;
\State $\s{x}_{t-1} = \frac{1}{\sqrt{\alpha_t}} \left(\hat{\s{x}}_t + \beta_t s_\theta (\hat{\s{x}}_t, t) \right)$.    
\State \underline{\emph{Inexact update of $\s{z}$ by gradient descent}}
\For{$k=0, 1, \dots, K_t$}  
\State $\s{z}_t^{(k+1)} = \s{z}_t^{(k)} - \eta \rho (\s{z}_t^{(k)} - \s{x}_t - \m{\nu}_t) $ \\
\qquad \qquad \qquad \qquad \qquad $ + \eta \nabla_{\s{z}} \log c_{\theta}(\tilde{\s{z}}_0({\s{z}}^{(k)}), \s{y})$;

\EndFor
\State $\s{z}_{t-1} = \s{z}_t^{(K_t+1)}$;

\State \underline{\emph{Update of the dual variable $\m{\nu}$}}
\State $\m{\nu}_{t-1} = \m{\nu}_t + \rho (\s{x}_{t-1} - \s{z}_{t-1})$;
\EndFor
\end{algorithmic}
\end{algorithm}

Compared to the standard ADMM method, Algorithm \ref{alg:admm} offers greater flexibility via solving the minimization subproblem inexactly. 
However, analyzing the convergence of this ADMM-based algorithm encounters a significant challenge due to that the non-convex nature of $\log q_{\m{\phi}} (\s{x})$ and $\log c_{\m{\theta}} (\s{y} | \s{z})$ hinders the usage of the standard variational inequality framework for the ADMM in this context~\citep{boyd_2011_distributed}.
In this paper, we facilitate that the subproblems become convex programming problems when the hyperparameter $\rho$ is sufficiently small.
Based on this observation, we can prove the convergence of our proposed algorithm under some mild assumptions on the objective functions.

{\bf Sufficient decreasing property.}  
Proposition \ref{prop:1} states that the proximal operator of negative log-likehood $-\rho \log q_{\m{\phi}}(\s{x})$ can be well approximated by using the diffusion reverse process. 
Then, with a proper assumption, we can derive the sufficient decreasing property in the following. 

\begin{theorem} \label{thm:1}

Assume that (1) there exists a constant $\delta > 0$, such that $| \log p(\s{x}) - \log q_{\m{\phi}}(\s{x}) | \leq \frac{\delta}{2}$ holds for any $\s{x}$; and (2) $\shorthyphen\log q_{\m{\phi}} (\s{x})$ is $L$-smooth. 
Then, there exists $\delta_t > 0$, such that 
\begin{equation*}
\begin{aligned}
    & - \log p (\s{x}_{t-1}) + \frac{1}{2 \rho} \| \s{x}_{t-1} - {\s{x}}_t \|^2 \\
    & \qquad \leq \min_{\s{x}} \left\{ - \log p (\s{x}) + \frac{1}{2 \rho} \| \s{x} - {\s{x}}_t \|^2 \right\} + \delta_t.
\end{aligned}
\end{equation*}
    
\end{theorem}

\emph{Remarks.} The detailed formation of $\delta_t$ is provided in Appendix \ref{sec: appendix2}.  Theorem~\ref{thm:1} essentially states that the gap between the two proximal operators of $-\log q_{\phi}(\s{x})$ and $-\log p(\s{x})$ can be properly bounded if the ground-truth distribution $ p(\s{x})$ can be well estimated by the variational distribution $q_{\m{\phi}}(\s{x})$ based on the diffusion model. 
This result paves the way for the sequel convergence analysis of Algorithm~\ref{alg:admm}. 

{\bf Convergence analysis.} Now we are ready to establish the convergence analysis of Algorithm \ref{alg:admm}. 
By combining the result of Theorem~\ref{thm:1} and some mild assumptions of smoothness, we can derive the theoretical results as follows. 

\begin{theorem} \label{thm: 2}

Assume that $\log q_{\phi} (\s{x})$ and $\log c_{\m{\theta}} (\s{z}; \s{y})$ are both $L$-smooth, $\sum_t \delta_t < + \infty$, and $\rho \leq \frac{1}{6 L}$.  
Let the sequence generated by Algorithm \ref{alg:admm} be $\{ (\s{x}_t, \s{z}_t, \m{\nu}_t) \}_{t=T}^0$.  
If $\eta \leq \frac{1}{\rho + L}$ and $\sum_t 2^{-K_t} < + \infty$, then $(\s{x}_0, \s{z}_0)$ converges to the stationary point of problem (\ref{equ:primal-dual})
\begin{equation*}
    \lim_{T \to \infty} \| \s{z}_{1} - \s{z}_0 \| = 0, \quad \lim_{T \to \infty} \| \s{x}_1 - \s{x}_0 \| = 0, 
\end{equation*}
with a convergent rate 
\begin{equation*}
    \min_{j \in \{ T, \dots, 0\} } \left\{ \| \s{z}_{j+1} - \s{z}_{j} \|^2 + \| \s{x}_{j+1} - \s{x}_{j} \|^2 \right\} = o \left( \frac{1}{T} \right).
\end{equation*}

\end{theorem}

\emph{Remarks.} A more detailed version and the corresponding proof are provided in Appendix \ref{sec: appendix2}. 
The theorem shows that Algorithm~\ref{alg:admm} can successfully converges to the stationary point with a sublinear rate. 
The basic idea to prove it is to construct a suitable Lyapunov function which metricizes the primal residual and the dual residual.

\section{Experiments}
\label{sec:experiments}

We evaluate the efficacy of our method on three tasks using different diffusion models and guidance conditions. Namely, we evaluate on image synthesis tasks using both non-linear and linear guidance functions and controllable motion generation tasks using trajectory guidance.

\subsection{Image Generation with Non-linear Guidance}

We first apply our method to guide unconditional image diffusion models pretrained on CelebA-HQ human face dataset~\citep{meng2021sdedit}. We employ three different guidance conditions: segmentation map, sketch, and text. The experimental settings are consistent with~\citep{yu2023freedom}. For the segmentation map guidance, we use BiSeNet~\citep{yu2018bisenet} to compute the face parsing maps. For the sketch guidance, we use pretrained sketch generator from AODA~\citep{xiang2022adversarial}. For textual guidance, we leverage CLIP-B/16 text and image encoders to get the embeddings.

We compare the proposed method with four baseline methods, i.e., Diffusion Posterior Sampling (DPS)~\citep{chung2022diffusion}, Loss-Guided Diffusion with Monte Carlo (LGD-MC)~\citep{song2023loss}, Training-Free Energy-Guided Diffusion Models (FreeDoM)~\citep{yu2023freedom}, and Manifold Preserving Guided Diffusion (MPGD)~\citep{he2023manifold}. For the LGD-MC, we set the size of Monte-Carlo sampling by $n=10$. We adopt Algorithm 1 and 3 of \citep{he2023manifold} and apply pixel diffusion models with VQGAN's autoencoder when evaluating MPGD.

For each condition type from segmentation map, sketch, and text prompt, we select 1000 conditions for evaluation. Table~\ref{table:face_dist_fid} compares $\ell_2$ distances and FID scores of our method with the baselines. 
When evaluating the guidance of text prompts, we additionally compute the CLIP scores in Table~\ref{table:face_clip} as a more generic metric to evaluate text-image alignment. 
The results demonstrate that \OurMethod{} achieves state-of-the-art performance on non-linear guided image generation tasks in terms of both image quality and guidance following. 

\begin{figure*}[t]
    \centering
    \includegraphics[width=0.84\linewidth]{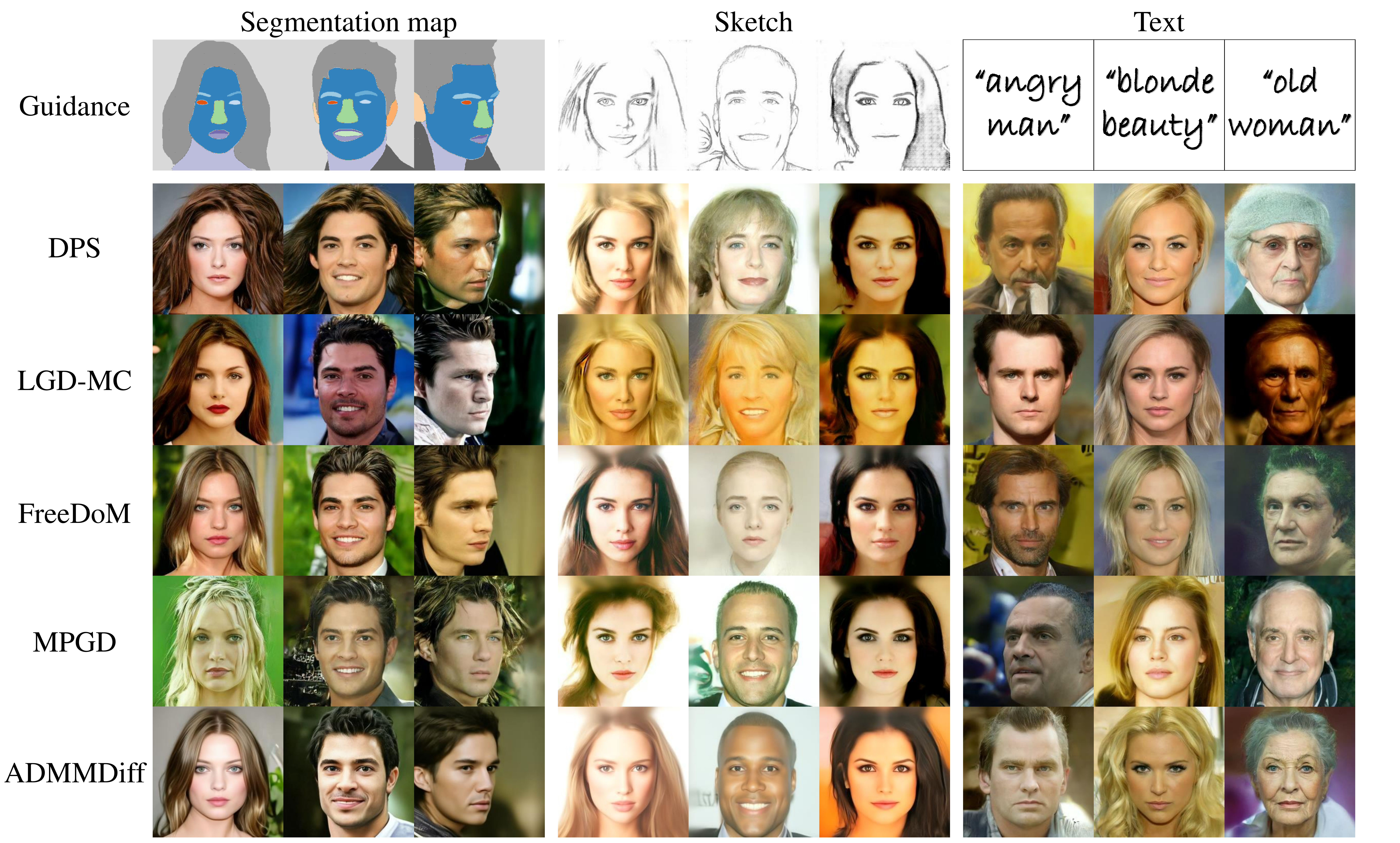}
    \caption{Qualitative comparison on CelebA-HQ in three conditional image synthesis tasks:  (a) segmentation maps to human faces; (b) sketches to human faces; (c) text prompts to human faces. Our method offers comparable image quality and advantage in the degree of satisfaction of the conditions.}
    \label{fig:face}
\end{figure*}

\begin{table*}[t]
\centering
\resizebox{0.9\textwidth}{!}{
\begin{tabular}{lllllllllll}
\toprule
{} & \multicolumn{2}{c}{\textbf{SR ($\times 4$)}} & \multicolumn{2}{c}{\textbf{Inpaint (box)}} & \multicolumn{2}{c}{\textbf{Inpaint (random)}} &
\multicolumn{2}{c}{\textbf{Deblur (gauss)}} & \multicolumn{2}{c}{\textbf{Deblur (motion)}}\\
\cmidrule(lr){2-3}
\cmidrule(lr){4-5}
\cmidrule(lr){6-7}
\cmidrule(lr){8-9}
\cmidrule(lr){10-11}
{\textbf{Method}} & {FID $\downarrow$} & {LPIPS $\downarrow$} & {FID $\downarrow$} & {LPIPS $\downarrow$} & {FID $\downarrow$} & {LPIPS $\downarrow$} & {FID $\downarrow$} & {LPIPS $\downarrow$} & {FID $\downarrow$} & {LPIPS $\downarrow$}\\
\midrule

\rowcolor[gray]{.9} \multicolumn{11}{c}{\sc FFHQ dataset} \\

ADMM-TV & 110.6 & 0.428 & 68.94 & 0.322 & 181.5 & 0.463 & 186.7 & 0.507 & 152.3 & 0.508 \\
Score-SDE~\citep{song_2020_score} & 96.72 & 0.563 & 60.06 & 0.331 & 76.54 & 0.612 & 109.0 & 0.403 & 292.2 & 0.657 \\
PnP-ADMM~\citep{chan2016plug} & 66.52 & 0.353 & 151.9 & 0.406 & 123.6 & 0.692 & 90.42 & 0.441 & 89.08 & 0.405 \\
MCG~\citep{chung2022improving} & 87.64 & 0.520 &40.11 & 0.309 & 29.26 & 0.286 & 101.2 & 0.340 & 310.5 & 0.702 \\
DDRM~\citep{kawar2022denoising} & 62.15 & 0.294 & 42.93 & 0.204 & 69.71 & 0.587 & 74.92 & 0.332 & - & -\\
DPS~\citep{chung2022diffusion} & \underline{39.35} & \underline{0.214} & \underline{33.12} & \underline{0.168} & \textbf{21.19} & \underline{0.212} & \underline{44.05} & \underline{0.257} & \underline{39.92} & \underline{0.242} \\
\textbf{ADMMDiff} & \textbf{26.8} & \textbf{0.203} & \textbf{24.97} & \textbf{0.160} & \underline{22.04} & \textbf{0.111} & \textbf{27.98} & \textbf{0.251} &\textbf{25.11} & \textbf{0.239}\\

\bottomrule

\rowcolor[gray]{.9} \multicolumn{11}{c}{\sc ImageNET dataset} \\

ADMM-TV & 130.9 & 0.523 & 87.69 & 0.319 & 189.3 & 0.510 & 155.7 & 0.588 & 138.8 & 0.525 \\
Score-SDE~\citep{song_2020_score} & 170.7 & 0.701 & 54.07 & 0.354 & 127.1 & 0.659 & 120.3 & 0.667 & 98.25 & 0.591 \\
PnP-ADMM~\citep{chan2016plug} & 97.27 & 0.433 & 78.24 & 0.367 & 114.7 & 0.677 & 100.6 & 0.519 & 89.76 & 0.483 \\
MCG~\citep{chung2022improving} & 144.5 & 0.637 & 39.74 & 0.330 & 39.19 & 0.414 & 95.04 & 0.550 & 186.9 & 0.758 \\
DDRM~\citep{kawar2022denoising} & 59.57 & 0.339 & 45.95 & \textbf{0.245} & 114.9 & 0.665 & 63.02 & \underline{0.427} & - & -\\
DPS~\citep{chung2022diffusion} & \underline{50.66} & \underline{0.337} & \underline{38.82} & \underline{0.262} & \underline{35.87} & \underline{0.303} & \underline{62.72} & 0.444 & \underline{56.08} & \underline{0.389} \\
\textbf{ADMMDiff} & \textbf{49.97} & \textbf{0.331} & \textbf{37.47} & \underline{0.262} & \textbf{34.80} & \textbf{0.207} & \textbf{51.3} & \textbf{0.414} &\textbf{55.99} & \textbf{0.364}\\

\bottomrule
\end{tabular}
}
\caption{
Quantitative evaluation (FID, LPIPS) of solving linear inverse problems on FFHQ 256$\times$256-1k validation dataset and ImageNet 256$\times$256-1k validation dataset. 
\textbf{Bold} indicates the best. \underline{Underline} indicates the second best.
}
\vspace{-0.5em}
\label{tab:nip_ffhq_fid_lpips}
\end{table*}

\begin{figure*}[t!]
    \centering
    \includegraphics[width=0.95\linewidth]{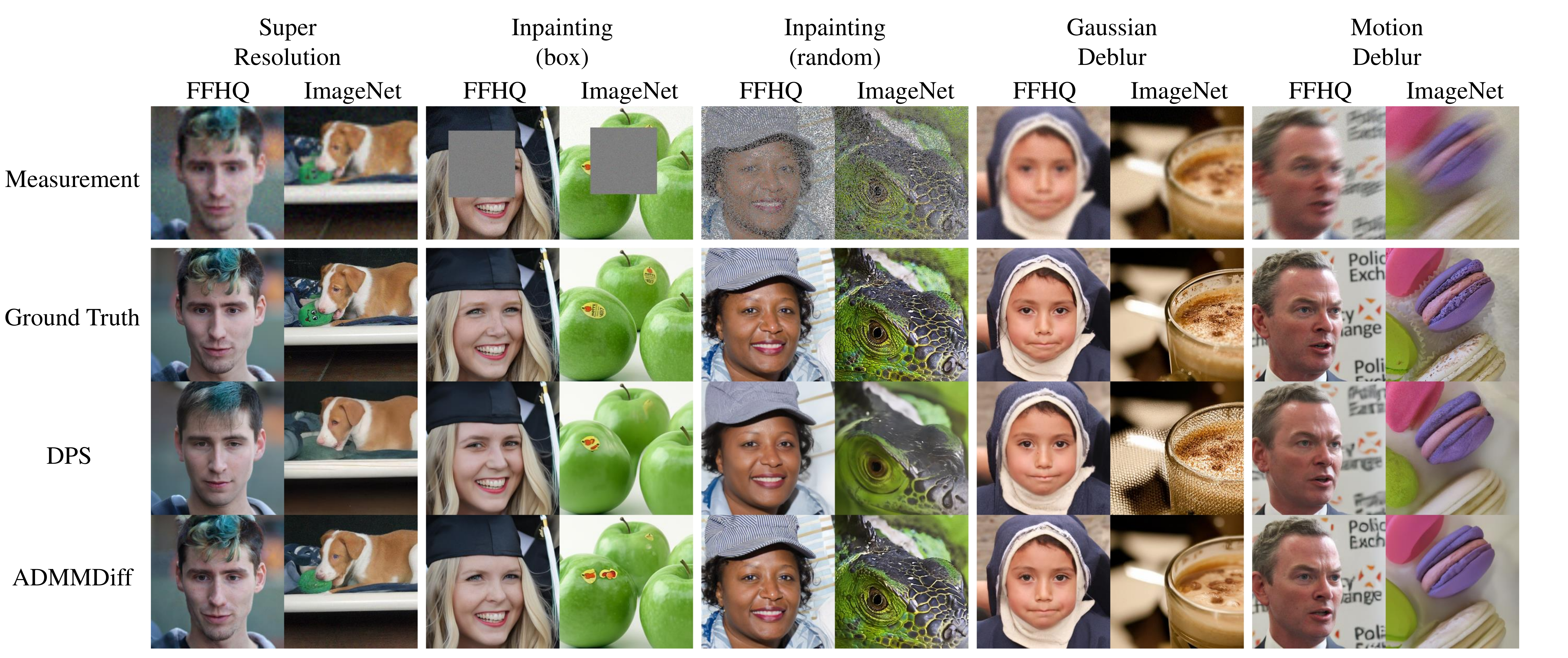}
    \caption{Qualitative comparison on solving linear inverse problems using our method and DPS. Our method consistently achieve comparable image quality and advantage in the degree of satisfaction of the conditions.}
    \label{fig:linear}
\end{figure*}

\begin{table*}[t]
    \centering
    \resizebox{0.8\textwidth}{!}{
    \begin{tabular}{c | cc | cc | cc| cc}
        \toprule
          \multicolumn{1}{c}{} \vline & \multicolumn{2}{c}{``walking''} \vline & \multicolumn{2}{c}{``raising hands''}\vline & \multicolumn{2}{c}{``jogging around''}\vline & \multicolumn{2}{c}{``backwards''}\\
          \multicolumn{1}{c}{\textbf{Method}} \vline & Obj. ↓ & Emb. ↓ & Obj. ↓ & Emb. ↓ & Obj. ↓  & Emb. ↓ & Obj. ↓ & Emb. ↓ \\
        \hline
        \rule{0pt}{10pt}
        Unconditional~\citep{tevet2023human} & 7.139 & 1.48 & 3.367 & 7.80 & 8.357 & 5.17 & 3.344 & 4.96  \\\rule{0pt}{10pt}
        GMD w/o inpainting~\citep{karunratanakul2023guided} & 0.095 & 4.62 & 0.297 & 7.93 & 0.104 & 5.54 & 0.223 & 8.94 \\\rule{0pt}{10pt}
        LGD-MC~\citep{song2023loss} & 0.107 & 4.67 & 0.272 & 8.13 & 0.109 & 5.45 & 0.237 & 8.83  \\\rule{0pt}{10pt}
        \textbf{ADMMDiff} & \textbf{0.042} & \textbf{3.38} & \textbf{0.064} & \textbf{7.28} & \textbf{0.057} & \textbf{5.23} & \textbf{0.050} & \textbf{8.30} \\
        \bottomrule
    \end{tabular}}
    \caption{Comparison on controllable motion generation task using trajectory guidance. Our method achieves the best guidance performance with lowest objective value. In terms of textual alignment, the embedding distance of our method is the closest to the unconditional model and outperforms the two guided motion diffusion baselines. \textbf{Bold} indicates the best.}
    \label{table:motion_loss_clip}
\end{table*}

\subsection{Linear Inverse Problems}

We next evaluate our method on linear inverse problems, instantiated by super-resolution, inpainting and deblurring tasks. 
The experimental settings follow~\citep{chung2022diffusion}. 
We select the FFHQ validation set~\citep{karras2019style} and the ImageNet validation set~\citep{deng2009imagenet} as two benchmark datasets. 
We adopt two unconditional diffusion models pre-trained from~\citep{chung2022diffusion,dhariwal_2021_diffusion} for these two datasets, respectively.
For linear inverse problems which have a measurement function of the form $\s{y}=\m{A}\s{x}+\m{\epsilon}$, we compute the loss as $\ell_2 = \|\m{A}\hat{\s{x}}_0 - \s{y}\|_2^2$,
where $\hat{\s{x}}_0$ represents the estimated clean image at each reverse step and $\s{y}$ is the provided noisy measurement depending on the task. We compare our method with various baselines including diffusion posterior sampling (DPS)~\citep{chung2022diffusion}, denoising diffusion restoration models (DDRM)~\citep{kawar2022denoising}, manifold constrained gradients (MCG)~\citep{chung2022improving}, plug-and-play alternating direction method of multipliers (PnP-ADMM)~\citep{chan2016plug}, total-variation sparsity regularized optimization method (ADMM-TV) and Score-SDE~\citep{song_2020_score}.

On both FFHQ and ImageNet validation set, we test on 1000 images using perceptual metrics including Fréchet Inception Distance (FID) and Learned Perceptual Image Patch Similarity (LPIPS) distance, and standard distortion metrics including peak-signal-to-noise-ratio (PSNR) and structural similarity index (SSIM). In Table~\ref{tab:nip_ffhq_fid_lpips}, we compare our method with baselines on FID and LPIPS. 
Among the $10$ tasks across the two dataset. 
The results show that our methods achieve state-of-the-art performances on FID score for 9 out of 10 tasks, LPIPS for 9 out of 10 tasks. 
We also report PSNR and SSIM of our method and baselines in Appendix~\ref{app:exp}. 
Though our method is not specially designed to solve linear inverse problems, it still outperforms most of the baselines on different tasks. In Figure~\ref{fig:linear}, we further provide visual comparison of our method against single trajectory method DPS. 
It can be observed that, compared with DPS which may output blurry images, our proposed method stably gives results with superior image quality and fidelity to original images.

\subsection{Guided Motion Generation}

In this section, we extend our algorithm to controllable human motion generation task. We use Guided Motion Diffusion (GMD)~\citep{tevet2023human} without guidance as the base motion diffusion model. To control the general behavior of human motion, we employ the text prompt input as prior and test on fine-grained trajectory guidance. 
In specific, we manually configure the temporal motion trajectory and guide the motion generation to follow the trajectory. The loss of the trajectory guidance is defined by the sum of the $\ell_2$ distances between $\text{root}(\hat{\s{x}}_0^i)$ and $y^i$. For a given time step $t$, the loss is defined as $\ell_2 = \sum_{i<N} \|\s{y}^i - \text{root}(\hat{\s{x}}_0^i)\|_2^2$
where $\s{y}^i$ is the $i$-th position in the trajectory projected on the ground plane and $\text{root}(\hat{\s{x}}_0^i)$ is the root position of the $i$-th motion. The overall objective is to align the semantic of the text prompt while minimizing the loss function to follow the configured trajectory. Following~\citep{song2023loss}, we evaluate the performance using the \textit{Objective} and \textit{Embedding} metrics. The \textit{Objective} metric measures the loss function $\ell_2$ defined above. The \textit{Embedding} metric measures the $\ell_2$ distance between the motion embedding and the text embedding from the pretrained text-motion encoders of T2M~\citep{guo2022generating}. The validity of using the embedding distance to reflect text-motion alignment is demonstrated in Appendix B.3.1 in~\citep{song2023loss}.

To evaluate the metrics, we choose four prompts: (1) ``a person is walking''; (2) ``a person is walking while raising both hands''; (3) ``a person is jogging while turning around''; (4) ``a person is walking backwards''. For each text prompt, we sample 10 trajectories as path conditions using the trajectory diffusion model in GMD~\citep{karunratanakul2023guided} and calculate the average value of each metric. We compare our method with GMD~\citep{karunratanakul2023guided} and Loss-Guided Diffusion with Monte Carlo (LGD-MC)~\citep{song2023loss}. For LGD-MC, we set the size of Monte-Carlo sampling by $n=10$. For GMD with guidance, we remove the trajectory inpainting post-processing for fair comparison. The results in table~\ref{table:motion_loss_clip} show that our method achieves both the lowest embedding and objective distance compared with two guided motion generation baselines. Figure~\ref{fig:motion} compares the visual results of our method with GMD, where our method is shown to better follow the designed motion trajectory. While previous work~\citep{song2023loss} has already shown that baseline methods can do simple path following like straight lines, our visualization proves the superiority of our method in following more complicated trajectories. The comparison demonstrates that our method can satisfy the guidance when it is significantly different from the motion diffusion prior and meanwhile preserve the high-level semantics entailed by the motion diffusion prior.

\section{Related Works}
\label{sec:related}

\subsection{Training-Free Guided Diffusion}
Diffusion-based training-free guidance methods~\citep{graikos2022diffusion,chung2022diffusion,song2023loss,yu2023freedom,bansal2023universal,kawar2022denoising,chung2022improving,karunratanakul2023guided,song_2020_score,dhariwal_2021_diffusion} aims to leverage unconditional diffusion models for conditional generation. A special case of the guidance tasks is linear inverse problems, where the guidance loss is derived from a linear forward model. A serious of methods~\citep{song_2020_score,chan2016plug,chung2022improving,chung2022diffusion,dou2024diffusion}, are specially designed to solve such linear guidance. Other guidance tasks use general differential loss functions as guidance. Methods target at general loss guidance~\citep{song_2020_score,dhariwal_2021_diffusion,chung2022diffusion,song2023loss,yu2023freedom,he2023manifold,bansal2023universal}, aims to generate samples that satisfy both the prior diffusion of unconditional diffusion models and the guidance alignment measured by the loss of guidance functions. 

Guiding unconditional diffusion models typically requires the estimation of $\nabla_{\s{x}_t}\log p_t(\s{y}|\s{x}_t)$ at each denoising step. Early works directly measure the gradient by training time-dependent classifiers or loss functions. Such classifiers or loss functions are task-dependent and hard to acquire. To improve the flexibility and leverage off-the-shelf guidance functions pretrained on clean samples, DPnP~\citep{graikos2022diffusion} proposes to leverage the score-matching as a plug-and-play objective to achieve the alignment with diffusion prior while minimizing the loss function for guidance satisfication. DPS~\citep{chung2022diffusion} measures the intractable term $\nabla_{\s{x}_t}\log p_t(\s{y}|\s{x}_t)$ through an intermediate estimation of $\hat{x}_0$ at each denoising step. A few works follow the sampling-time guidance methods introduced by DPS. Among them, LGD-MC~\citep{song_2020_score} improves the point estimation of $\hat{x}_0$ with Monte Carlo. UGD~\citep{bansal2023universal} and FreeDoM~\citep{yu2023freedom} use time-travel strategy to solve the poor guidance problem in large domain. MPGD~\citep{he2023manifold} focuses on the issue of manifold preservation and leverages pretrained autoencoders or latent space to make generated samples lie on the clean data manifold. Recent research~\citep{ye2024tfg} also shows that most existing approaches aim to guide a single reverse diffusion trajectory can be unified by a framework with different choices of hyperparameters. Compared with them, our method evolves two trajectories with constraints to balance the objectives of prior and guidance satisfication.

\subsection{ADMM for Image-Based Tasks}
Several works~\citep{chan2016plug, wang2017parameter,yazaki2019interpolation, chan2019performance} have explored the application of ADMM in image-based tasks, leveraging its ability to handle complex optimization problems efficiently. These methods primarily build upon the PnP-ADMM~\citep{chan2016plug}, which introduces a flexible framework where off-the-shelf image denoising algorithms are used as proximal operators within ADMM for image restoration tasks. In contrast, our approach is the first to incorporate diffusion models as the proximal operator for $\s{x}_t$, taking advantage of their generative capabilities. Furthermore, we utilize Tweedie’s formula to estimate $\s{z}_0$, enabling our method to handle various conditional generation tasks beyond traditional restoration problems. This combination not only extends the applicability of ADMM to generative modeling but also ensures high-quality outputs by effectively balancing data fidelity and conditional alignment. Additionally, we provide a rigorous convergence analysis for the proposed method, offering new theoretical insights into its performance. This analysis highlights the robustness and efficiency of our approach, establishing a solid foundation for its application in diverse image-based generation tasks.

\section{Conclusion and Future Work}
\label{sec:conclusion}
In this paper, we propose a novel framework for training-free guided diffusion that decouples the unconditional diffusion model from the guidance function, introducing a new ADMM-based algorithm to effectively solve and couple the resulting subproblems. Our method, \OurMethod{}, demonstrates a strong theoretical foundation with convergence guarantees and consistently outperforms existing approaches across various tasks, including guided image generation and motion synthesis. 
In future work, we plan to explore techniques to further accelerate the sampling process, enabling our method more effective and efficient for real-time applications. Additionally, we aim to incorporate more guidance functions to handle more complex and diverse conditional generation tasks.

\newpage
\bibliographystyle{ieeenat_fullname}
\bibliography{main}

\newpage
\appendix
\onecolumn

\section{Theoretical Analysis} \label{sec: appendix2}

{\bf Assumptions.} To derive the sequel propositions and theorems, we first introduce a general assumption about the smoothness of $- \log q_{\m{\phi}} (\s{x})$ and $- \log c_{\m{\theta}} (\s{z, y})$, that is, these two log-likelihood functions are assumed to be continuous differentiable, and their gradients are $L$-Lipschitz continuous.

\subsection{Proof of Theorem \ref{thm:1}}
First, we show that the gap between the proximal operators is bounded above.
\begin{proof}[\bf{Proof of Theorem \ref{thm:1}}]
For the one diffusion reverse step, the hyperparameter $\delta_t$ is equal to $\delta + \frac{L}{2} \left( \frac{Q_f - 1}{Q_f + 1} \right)^2 \| \s{x}_t - \tilde{\s{x}}_0 \|^2$.
Because $-\log q_{\phi}$ is $L$-smooth, then we know that $\s{x} \to - \log q_{\phi} (\s{x}) + \frac{\rho}{2} \| \s{x} - \s{\hat{x}}_t \|^2$ is $(\rho-L)$-strongly convex and $(\rho + L)$-smooth. 
Note that $\s{x}_{t-1} = \s{x}_t - \eta \nabla h_t (\s{x}_t)$, then by the standard convex optimization theory \citep{nesterov_2018_lectures}, we have 
\begin{equation} \label{equ: a8}
    - \log q_{\phi} (\s{x}_{t-1}) + \frac{\rho}{2} \| \s{x}_{t-1} - \hat{\s{x}}_t \|^2 \leq  - \log q_{\phi} (\hat{\s{x}}) + \frac{\rho}{2} \| \hat{\s{x}} - \hat{\s{x}}_t \|^2  + \frac{L}{2} \left( \frac{Q_f - 1}{Q_f + 1} \right)^2 \| \s{x}_t - \tilde{\s{x}}_0 \|^2,
\end{equation}
where $Q_f = (\rho - L) / (\rho + L)$.
Combining \eqref{equ: a8} and part (\romannumeral1), we get that 
\begin{equation*}
    - \log p (\s{x}_{t-1}) + \frac{\rho}{2} \| \s{x}_{t-1} - \hat{\s{x}}_t \|^2 \leq  - \log p (\hat{\s{x}}) + \frac{\rho}{2} \| \hat{\s{x}} - \hat{\s{x}}_t \|^2 + \delta + \frac{L}{2} \left( \frac{Q_f - 1}{Q_f + 1} \right)^2 \| \s{x}_t - \tilde{\s{x}}_0 \|^2. \qedhere
\end{equation*}
    
\end{proof}

\subsection{Proof of Theorem \ref{thm: 2}}
Before proving the convergence results of Algorithm \ref{alg:admm}, we first discuss what the limit points are if Algorithm \ref{alg:admm} converges.

\begin{proposition} \label{prop: 5}

Let $ \{ (\s{x}_t, \s{z}_t, \m{\mu}_t) \}$ be the sequence generated by Algorithm \ref{alg:admm}.
If $\s{z}_t = \s{z}_{t-1}, \s{x}_t = \s{x}_{t-1}, \m{\mu}_t = \m{\mu}_{t-1}$, holds for some $t$, then 
\begin{equation*}
    \s{x}_{t-1} = \s{z}_{t-1}, \quad \nabla \left( \log p(\s{x_{t-1}}) + \log p (\s{y | x_{t-1}}) \right) = 0,
\end{equation*}
Hence, $\s{x}_{t-1}$ is a stationary point of the minimization problem.
    
\end{proposition}

\begin{proof}

According to the definition of $\m{\mu}_t$, we have 
\begin{equation*}
    \s{x}_{t-1} - \s{z}_{t-1} = \rho (\m{\mu}_{t-1} - \m{\mu}_t) = 0.
\end{equation*}
On the other side, $\s{x}_{t-1} \in \argmin_{\s{x}} - \log p (\s{x}) + \langle \s{x} - \s{z}_t, \m{\mu}_t \rangle + \frac{ \rho}{2} \| \s{x} - \s{z}_t \|^2$. Then from the first order condition, we obtain that 
\begin{equation} \label{equ: 5.21}
    \nabla_{\s{x}_{t-1}} \left( - \log p(\s{x}_{t-1}) + \m{\mu}_t + \rho (\s{x}_{t-1} - \s{z}_t) \right) = 0.
\end{equation}
Similarly, according to the definition of $\s{z}_{t-1}$, we know that 
\begin{equation} \label{equ: 5.22}
    \nabla_{\s{z}_{t-1}} \left( - \log p(\s{y} | \s{z}_{t-1}) - \m{\mu}_t + \rho (\s{x}_{t-1} - \s{z}_{t-1}) \right) = 0.
\end{equation}
Combining \eqref{equ: 5.21} and \eqref{equ: 5.22}, we get that 
\begin{equation*}
\nabla \left( - \log p(\s{x_{t-1}}) - \log p (\s{y | x_{t-1}}) \right) = 0. \qedhere
\end{equation*}
    
\end{proof}

To prove Theorem \ref{thm: 2}, we first present a simplified version of the Robbins-Siegmund theorem which will be used later. 

\begin{lemma}[\citep{robbins_1971_convergence}]

Consider a filter $\{ \mathcal{F}_k \}_k$, the nonnegative sequence of $\{ \mathcal{F}_k \}_k$-adapted processes $\{ V_k \}_k$, $\{ U_k \}_k$, and $\{ Z_k \}_k$ such that $\sum_k Z_k < + \infty$ almost surly, and 
\begin{equation*}
    \mathbb{E} [V_{k+1} | V_k ] + U_{k+1} \leq V_k + Z_k, \quad \forall k \geq 0.
\end{equation*}
Then $\{ V_k \}_k$ converges and $\sum_k U_k < + \infty$ almost surly.
    
\end{lemma}

The Robbins-Siegmund theorem provided in this paper is a stochastic version. 
However, we will only consider the deterministic version in the proof of Theorem \ref{thm: 2}.
We now move on to the proof of Theorem \ref{thm: 2}.

\begin{proof}[\bf{Proof of Theorem \ref{thm: 2}}]

To adapt the formal index settings in optimization perspective, we reverse the index order in the proof. 
In other words, we let $\s{x}_k \gets \s{x}_{T-k}$,  $\s{z}_k \gets \s{z}_{T-k}$, and $\m{\mu}_k \gets \m{\mu}_{T-k}$, respectively. 
Moreover, to simplify the notations used in the proof, we defined $g(\s{z}) := - \log c_{\theta} (\s{z, y})$ and $f(\s{x}) := - \log p (\s{x})$.

\noindent {\bf{Step 1}}.
Let $\s{z}_{k+1}^* := \argmin_\s{z} \l (\s{x}_{k+1}, \s{z}, \m{\mu}_k)$.
According to the assumption of $g$, we have that the mapping $\s{z} \to \l (\s{x}_{k+1}, \s{z}, \m{\mu}_k)$ is $(\rho - L)$-strongly convex and $(\rho + L)$-smooth. 
Thus the gradient descent steps in Algorithm \ref{alg:admm} show that 
\begin{equation*}
\left\|\s{z}_{k+1}-\s{z}_{k+1}^*\right\| \leqslant\left(1-\frac{\rho-L}{\rho+L}\right)^{K_k}\left\|\s{z}_k-\s{z}_{k+1}^*\right\| \leqslant \sqrt{d} \left(\frac{2 L}{\rho+L}\right)^{K_k},
\end{equation*}
where $d$ is the dimension of $\s{z}$.
The last inequality holds due to the nature of $\s{z}_k$ being an image, which implies that $\s{z}_k$ belongs to the interval $[0, 1]^d$.
For ease of notation, we define $\Delta_k := \sqrt{d} \left(\frac{2 L}{\rho+L}\right)^{K_k}$.

On the other side, according to the definition of $\s{z}_{k+1}^*$, we have 
\begin{equation} \label{equ: 4.14}
    0 = \nabla g(\s{z}_{k+1}^*) - \m{\mu}_k + \rho (\s{x}_{k+1} - \s{z}_{k+1}^*).
\end{equation}
Noting that $\m{\mu}_{k+1} = \m{\mu}_k + \rho (\s{x}_{k+1} - \s{z}_{k+1})$, we get  
\begin{equation*}
    \nabla g(\s{z}_{k+1}^*) = \m{\mu}_{k+1} + \rho (\s{z}_{k+1} - \s{z}_{k+1}^*),
\end{equation*}
from equation \eqref{equ: 4.14}.
Hence, we can obtain an upper bound of the gap between $\m{\mu}_{k+1}$ and $\m{\mu}_k$ by the smoothness of $g$, 
\begin{equation} \label{equ: 4.15}
\begin{aligned}
\left\|\m{\mu}_{k+1}-\m{\mu}_k\right\| & = \left\|\nabla g\left(\s{z}_{k+1}^*\right)-\nabla g\left(\s{z}_k^*\right)+\rho\left(\s{z}_{k+1}^*-\s{z}_{k+1}\right)-\rho\left(\s{z}_k^*-\s{z}_k\right)\right\| \\
& \leq  \left\|\nabla g\left(\s{z}_{k+1}^*\right)-\nabla g\left(\s{z}_k^*\right)\right\|+\rho\left\|\s{z}_{k+1}^*-\s{z}_{k+1}\right\|+\rho\left\|\s{z}_k^*-\s{z}_k\right\| \\
& \leq  L\left\|\s{z}_{k+1}^*-\s{z}_k^*\right\|+\rho\left\|\s{z}_{k+1}^*-\s{z}_{k+1}\right\|+\rho\left\|\s{z}_k^*-\s{z}_k\right\| \\
& \leq  L\left(\left\|\s{z}_{k+1}^*-\s{z}_{k+1}\right\|+\left\|\s{z}_{k+1}-\s{z}_k\right\|+\left\|\s{z}_k-\s{z}_k^*\right\|\right)+\rho\left\|\s{z}_{k+1}^*-\s{z}_{k+1}\right\|+\rho\left\|\s{z}_k^*-\s{z}_k\right\| \\
& =  L\left\|\s{z}_{k+1}-\s{z}_k\right\|+(L+\rho)\left\|\s{z}_{k+1}-\s{z}_{k+1}^* \right\|+(L+\rho)\left\| \s{z}_k-\s{z}_k^* \right\| .
\end{aligned}
\end{equation}

\noindent {\bf{Step 2}}.
Let $\s{x}_{k+1}^* := \argmin_\s{x} \l (\s{x}_, \s{z}_k, \m{\mu}_k)$.
Because $\s{x} \to \l ( \s{x}, \s{z}_{k}, \m{\mu}_k)$ is $(\rho - L)$-strongly convex for all $k \geq 0$, we have
\begin{equation*}
\frac{\rho -L}{2}\left\|\s{x}_{k+1}-\s{x}_{k+1}^*\right\|^2 \leqslant f\left(\s{x}_{k+1}\right)+\frac{\rho}{2}\left\|\s{x}_{k+1}-\s{u}_{k}\right\|^2-f\left(\s{x}_{k+1}^*\right)-\frac{\rho}{2}\left\|\s{x}_{k+1}^*-\s{u}_{k}\right\|^2 \leqslant \delta_k .
\end{equation*}
Hence, 
\begin{equation*}
    \left\|\s{x}_{k+1}-\s{x}_{k+1}^*\right\|^2 \leq \frac{2 \delta_k}{\rho - L} .
\end{equation*}
Then, 
\begin{equation*}
\begin{aligned}
\l\left(\s{x}_k, \s{z}_k, \m{\mu}_k\right) -\l\left(\s{x}_{k+1}, \s{z}_k, \m{\mu}_k\right) 
& \ge \langle\nabla_\s{x} \l \left(\s{x}_{k+1}, \s{z}_k, \m{\mu}_k\right), \s{x}_k-\s{x}_{k+1}\rangle+\frac{\rho-L}{2}\left\|\s{x}_k-\s{x}_{k+1}\right\|^2 \\
& \ge -\frac{1}{2}\left\|\nabla_\s{x} \l \left(\s{x}_{k+1}, \s{z}_k, \m{\mu}_k\right)\right\|^2-\frac{1}{2}\left\|\s{x}_k-\s{x}_{k+1}\right\|^2+\frac{\rho-L}{2}\left\|\s{x}_k-\s{x}_{k+1}\right\|^2 \\
& \ge -\frac{(\rho+L) \delta_k}{\rho - L} +\frac{\rho-L-1}{2}\left\|\s{x}_k-\s{x}_{k+1}\right\|^2,
\end{aligned}
\end{equation*}
where the last inequality holds because $\s{x} \to \l (\s{x} , \s{z}_k, \,\m{\mu}_k)$ is $(\rho + L)$-smooth, and $\s{x}_{k+1}^* := \argmin_\s{x} \l (\s{x} , \s{z}_k, \m{\mu}_k)$.

\noindent {\bf{Step 3}}. Now, we have
\begin{equation*}
\begin{aligned}
& \l \left(\s{x}_{k+1}, \s{z}_{k}, \m{\mu}_k\right)-\l \left(\s{x}_{k+1}, \s{z}_{k+1}, \m{\mu}_{k+1}\right) \\
&\qquad = g\left(\s{z}_k\right)-g\left(\s{z}_{k+1}\right)+\left\langle\m{\mu}_{k+1}, \s{z}_{k+1}-\s{z}_k\right\rangle+\frac{\rho}{2}\left\|\s{z}_k-\s{z}_{k+1}\right\|^2- \rho \left\|\m{\mu}_{k}- \m{\mu}_{k+1}\right\|^2 \\
& \qquad = g\left(\s{z}_k\right)-g\left(\s{z}_{k+1}\right)+\left\langle\nabla g\left(\s{z}_{k+1}^*\right)-\rho \left(\s{z}_{k+1}-\s{z}_{k+1}^*\right), \s{z}_{k+1}-\s{z}_k\right\rangle +\frac{\rho}{2}\left\|\s{z}_k-\s{z}_{k+1}\right\|^2- \rho \left\|\m{\mu}_k-\m{\mu}_{k+1}\right\|^2 \\
& \qquad = g\left(\s{z}_k\right)-g\left(\s{z}_{k+1}\right)+\left\langle\nabla g\left(\s{z}_{k+1}\right), \s{z}_{k+1}-\s{z}_k\right\rangle+\left\langle\nabla g\left(\s{z}_{k+1}\right)-\nabla g\left(\s{z}_{k+1}^*\right), \s{z}_{k+1}-\s{z}_k\right\rangle \\
& \qquad \qquad +\rho \left\langle \s{z}_{k+1}^*-\s{z}_{k+1}, \s{z}_{k+1}-\s{z}_k\right\rangle+\frac{\rho}{2}\left\|\s{z}_k-\s{z}_{k+1}\right\|^2- \rho \left\|\m{\mu}_k-\m{\mu}_{k+1}\right\|^2 \\
& \qquad \geqslant -\frac{L}{2}\left\|\s{z}_k-\s{z}_{k+1}\right\|^2-\frac{1}{2}\left\|\s{z}_k-\s{z}_{k+1}\right\|^2-\frac{L^2}{2}\left\|\s{z}_{k+1}-\s{z}_{k+1}^*\right\|^2-\frac{1}{2}\left\|\s{z}_{k+1}-\s{z}_k\right\|^2 \\ 
& \qquad \qquad -\frac{ \rho^2}{2}\left\|\s{z}_{k+1}-\s{z}_{k+1}^*\right\|^2 +\frac{\rho}{2}\left\|\s{z}_k-\s{z}_{k+1}\right\|^2- \frac{1}{\rho} \left\|\m{\mu}_k-\m{\mu}_{k+1}\right\|^2 \\
& \qquad \geqslant \left(\frac{\rho}{2}-\frac{L}{2}-1\right) \left\| \s{z}_k-\s{z}_{k+1}\right\|^2-\frac{L^2+\rho^2}{2}\left\| \s{z}_{k+1}-\s{z}_{k+1}^*\right\|^2- \frac{1}{\rho} \left\| \m{\mu}_k-\m{\mu}_{k+1} \right\|^2 \\
& \qquad \geqslant \left(\frac{\rho}{2}-\frac{L}{2}-1\right)\left\|\s{z}_k-\s{z}_{k+1}\right\|^2-\frac{L^2+\rho^{2}}{2}\left\|\s{z}_{k+1}-\s{z}_{k+1}^*\right\|^2 \\ 
& \qquad \qquad - \frac{3}{\rho} \left(L^2\left\|\s{z}_{k+1}-\s{z}_k\right\|^2+ (L+ \rho)^2 \Delta_k^2 + (L+\rho)^2 \Delta_{k-1}^2 \right)^2 \\ 
& \qquad \geqslant \left(\frac{\rho}{2}-\frac{L}{2}-1- \frac{3 L^2}{\rho} \right)\left\|\s{z}_k-\s{z}_{k+1}\right\|^2 - \left( \frac{L^2 + \rho^2}{2} \Delta_k^2 + \frac{3 (L+\rho)^2}{\rho} \Delta_k^2 + \frac{3 (L+\rho)^2}{\rho} \Delta_{k-1}^2 \right).
\end{aligned}
\end{equation*}
{\bf{Step 4}}.
In this step, we will show that the augmented Lagrangian function sequence $\{ \l (\s{x}_k, \s{z}_k, \m{\mu}_k) \}_k$ is nonincreasing.
Let $\l_k := \l (\s{x}_k, \s{z}_k, \m{\mu}_k)$.
According to Step 2 and 3, we get that 
\begin{equation*}
\begin{aligned}
\l_k - \l_{k+1} 
\geqslant & \left(\frac{\rho}{2}-\frac{L}{2}-1- \frac{3 L^2}{\rho}\right)\left\|\s{z}_k-\s{z}_{k+1}\right\|^2
+\frac{\rho-L-1}{2}\left\|\s{x}_k-\s{x}_{k+1}\right\|^2 \\
& \qquad - \left( \frac{L^2 + \rho^2}{2} \Delta_k^2 + \frac{3 (L+\rho)^2}{\rho} \Delta_k^2 + \frac{3 (L+\rho)^2}{\rho} \Delta_{k-1}^2 \right) -\frac{(\rho+L) \delta_k}{\rho - L} .
\end{aligned}
\end{equation*}
Because $\rho > 6 L$, then there exist constants 
\begin{equation*}
     c_1 := L,  \quad c_2 := \frac{L^2 + \rho^2}{2} + \frac{3 (L+\rho)^2}{\rho}, \quad c_3 := \frac{3 (L+\rho)^2}{\rho}, \quad c_4 := \frac{(\rho+L)}{\rho - L},
\end{equation*}
such that 
\begin{equation} \label{equ: 4.22}
    c_1 \left( \| \s{z}_k - \s{z}_{k+1} \|^2 + \| \s{x}_k - \s{x}_{k+1} \|^2 \right) + \l_{k+1} 
    \leq \l_k + c_2 \Delta_k^2 + c_3 \Delta_{k - 1}^2 + c_4 \delta_k.
\end{equation}
Hence that the sequence $\{ \l\left(\s{x}_k, \s{z}_k, \m{\mu}_k\right) \}_k$ is decreasing.
In the rest of this step, we will show that $\{ \l_k \}_k$ is bounded from below.
Because we assume that $\min \{ f(\s{z}) + g(\s{x}): \s{z} = \s{x} \} > -\infty$, we obtain that for all $k$,
\begin{equation*}
\begin{aligned}
    \l (\s{x}_k, \s{z}_k, \m{\mu}_k; \rho) & = f(\s{x}_k) + g(\s{z}_k) + \langle \m{\mu}_k, \s{x}_k - \s{z}_k \rangle + \frac{\rho}{2} \| \s{z}_k - \s{x}_k \|^2 \\ 
    & = f(\s{x}_k) + g(\s{z}_k) + \langle \nabla g(\s{z}_k^*) - \rho (\s{z}_k - \s{z}_k^*), \s{x}_k - \s{z}_k \rangle + \frac{\rho}{2} \| \s{z}_k - \s{x}_k \|^2 \\ 
    & = f(\s{x}_k) + g(\s{x}_k) + g(\s{z}_k) - g(\s{x}_k) + \langle \nabla g(\s{z}_k), \s{x}_k - \s{z}_k \rangle + \frac{\rho}{2} \| \s{z}_k - \s{x}_k \|^2 \\ 
    & \qquad + \langle \nabla g(\s{z}_k^*) - \nabla g(\s{z}_k) - \rho (\s{z}_k - \s{z}_k^*), \s{x}_k - \s{z}_k \rangle \\ 
     & \geq f(\s{x}_k) + g(\s{x}_k) - \frac{L}{2} \| \s{x}_k - \s{z}_k \|^2 + \frac{\rho}{2} \| \s{z}_k - \s{x}_k \|^2 - \frac{L^2 + \rho^2}{2} \| \s{z}_k^* - \s{z}_k \|^2 - \| \s{x}_k - \s{z}_k \|^2 \\ 
    & \geq f(\s{x}_k) + g(\s{x}_k) - \frac{L^2 + \rho^2}{2} \| \s{z}_k^* - \s{z}_k \|^2 \geq \min \{ f(\s{z}) + g(\s{z}) \} - \frac{(L^2 + \rho^{2}) \Delta_{k-1}^2}{2}.
\end{aligned}
\end{equation*}
Hence, the sequence $\{ \l\left(\s{z}_k, \s{x}_k, \m{\mu}_k\right) \}_k$ is bounded from below.

\noindent {\bf{Step 5}}. 
Now we will prove that $\lim_{j \to \infty} (\| \s{z}_{k+1} - \s{z}_k \|^2 + \| \s{x}_{k+1} - \s{x}_k \|^2 ) = 0$.
For ease of notation, we define $m_k := \| \s{z}_{k+1} - \s{z}_k \|^2 + \| \s{x}_{k+1} - \s{x}_k \|^2$ for all $k = 0, 1, \dots$.
Summing equation \eqref{equ: 4.22} for $k$ from $0$ to $+\infty$, we get that 
\begin{equation*} \label{equ: 4.27}
    c_1 \sum_{k=0}^\infty  m_k \leq \l_0 - \inf_k \l_k + \sum_{k=0}^\infty \left( c_2 \Delta_k^2 + c_3 \Delta_{k - 1}^2 + {c_4} \delta_k \right). 
\end{equation*}
On the other side, recalling the definition of $\Delta_k$ and noting that $(2L) / (\rho + L) < 1$, we get that $\sum_k \Delta_k^2 < + \infty$. 
Therefore, the sum of $m_k$ for all $k$ is finite ($\sum_k m_k < + \infty$), which implies that $\lim_{k\to \infty} m_k = 0$.
Moreover, by \eqref{equ: 4.15}, we get that 
\begin{equation*}
    \| \m{\mu}_k - \m{\mu}_{k+1} \| \leq L\left\|\s{z}_{k+1}-\s{z}_k\right\|+ (L+\rho) (\Delta_k + \Delta_{k-1}) \to 0, \quad \text{as} ~ k \to \infty.
\end{equation*}

\noindent {\bf{Step 6}}.
In this part, we will show that the algorithm will converge to a stationary point of problem \eqref{equ:primal-dual}.
By Step 5, we know that 
\begin{equation*}
    \lim_k \| \s{x}_{k+1} - \s{x}_k \| = 0, \quad \lim_k \| \s{z}_{k+1} - \s{z}_k \| = 0, \quad \lim_k \| \m{\mu}_{k+1} - \m{\mu}_k \| = 0.
\end{equation*}
Let $(\s{x}_\infty, \s{z}_\infty, \m{\mu}_\infty)$ be the limit point of the sequence. 
Combing with Proposition \ref{prop: 5}, we get that $\s{x}_\infty$ is a stationary point of the minimization problem \eqref{equ:primal-dual}.

{\bf{Step 7}}.
In this step, we use the Robbins-Siegmund theorem to get a refine analysis to the convergence rate of the ADMM.
In particular, we prove the sublinear convergence rate.
Now define for all $k \in \mathbb{N}$, 
\begin{equation*}
    w_k := \frac{2}{k + 1}, \quad \lambda_0 := h_0, ~ \lambda_{k+1} := (1-w_k) \lambda_k + w_k m_k.
\end{equation*}
Note that if we show that $m_k \to 0$ as $k \to \infty$, we can prove that convergence of the ADMM.
Now we have $w_k \in [0, 1]$, and $\lambda_{k+1}$ is a convex combination of $\{ m_0, \dots, m_k \}$. 
Rearranging equation \eqref{equ: 4.22}, we have 
\begin{equation*}
    \frac{(k+1) c_1}{2} \lambda_{k+1} + \l_{k+1} + \frac{c_1}{2} \lambda_k \leq \frac{k c_1}{2} \lambda_k + \l_k + c_2 \Delta_k^2 + c_3 \Delta_{k - 1}^2 + c_4 \delta_k.
\end{equation*}
Using the Robbins-Siegmund theorem, we have the sequence $\{ k \lambda_k \}_k$ converges and that $\sum_k \lambda_k < + \infty$. 
In particular, it implies that $\lim_{k \to \infty} \lambda_{k} = 0$.
Note that $\lambda_k = \frac{1}{k} \cdot k \lambda_k$, then we have $\lim_{k \to \infty} k \lambda_k = 0$ because $\sum_k \lambda_k < + \infty$. 
Hence, 
\begin{equation*}
    \lambda_k = o \left( k^{-1} \right).
\end{equation*}

On the other side, $\lambda_k$ is a convex combination of $\{ m_0, \dots, m_k \}$, then we have 
\begin{equation*}
    \min_{j \in \{0, \dots, k \} } m_j \leq  \lambda_k = o \left( k^{-1} \right).
    \qedhere
\end{equation*}
    
\end{proof}

\section{Further Experimental Results} \label{app:exp}

We provide additional quantitative evaluations based on the PNSR and SSIM metrics in Table~\ref{tab:nip_ffhq_psnr_ssim} and Table~\ref{tab:nip_imagenet_psnr_ssim}.

\begin{table*}[htbp]
\centering
\resizebox{0.9\textwidth}{!}{
\begin{tabular}{lllllllllll}
\toprule
{} & \multicolumn{2}{c}{\textbf{SR ($\times 4$)}} & \multicolumn{2}{c}{\textbf{Inpaint (box)}} & \multicolumn{2}{c}{\textbf{Inpaint (random)}} &
\multicolumn{2}{c}{\textbf{Deblur (gauss)}} & \multicolumn{2}{c}{\textbf{Deblur (motion)}}\\
\cmidrule(lr){2-3}
\cmidrule(lr){4-5}
\cmidrule(lr){6-7}
\cmidrule(lr){8-9}
\cmidrule(lr){10-11}
{\textbf{Method}} & {PSNR $\uparrow$} & {SSIM $\uparrow$} & {PSNR $\uparrow$} & {SSIM $\uparrow$} & {PSNR $\uparrow$} & {SSIM $\uparrow$} & {PSNR $\uparrow$} & {SSIM $\uparrow$} & {PSNR $\uparrow$} & {SSIM $\uparrow$}\\
\midrule

ADMM-TV & 23.86 & 0.803 & 17.81 & 0.814 & 22.03 & 0.784 & 22.37 & 0.801 & 21.36 & 0.758 \\
Score-SDE~\citep{song_2020_score} & 17.62 & 0.617 & 18.51 & 0.678 & 13.52 & 0.437 & 7.12 & 0.109 & 6.58 & 0.102 \\
PnP-ADMM~\citep{chan2016plug} & 26.55 & \textbf{0.865} & 11.65 & 0.642 & 8.41 & 0.325 & \underline{24.93} & \textbf{0.812} & 24.65 & 0.825 \\
MCG~\citep{chung2022improving} & 20.05 & 0.559 & 19.97 & 0.703 & 21.57 & 0.751 & 6.72 & 0.051 & 6.72 & 0.055 \\
DDRM~\citep{kawar2022denoising} & 25.36 & 0.835 & 22.24 & 0.869 & 9.19 & 0.319 & 23.36 & 0.767 & - & -\\
DPS~\citep{chung2022diffusion} & \underline{25.67} & 0.852 & \underline{22.47} & \underline{0.873} & \underline{25.23} & \underline{0.851} & 24.25 & \underline{0.811} & 24.92 & \textbf{0.859} \\
\textbf{ADMMDiff} & \textbf{28.08} & \underline{0.857} & \textbf{22.5} & \textbf{0.878} & \textbf{33.4} & \textbf{0.930} & \textbf{25.1} & 0.794 &\textbf{25.6} & \underline{0.827}\\

\bottomrule
\end{tabular}
}
\caption{
Quantitative evaluation (PSNR, SSIM) of solving linear inverse problems on FFHQ 256$\times$256-1k validation dataset. \textbf{Bold} indicates the best. \underline{Underline} indicates the second best.
}
\label{tab:nip_ffhq_psnr_ssim}
\end{table*}

\begin{table*}[htbp]
\centering
\resizebox{0.9\textwidth}{!}{
\begin{tabular}{lllllllllll}
\toprule
{} & \multicolumn{2}{c}{\textbf{SR ($\times 4$)}} & \multicolumn{2}{c}{\textbf{Inpaint (box)}} & \multicolumn{2}{c}{\textbf{Inpaint (random)}} &
\multicolumn{2}{c}{\textbf{Deblur (gauss)}} & \multicolumn{2}{c}{\textbf{Deblur (motion)}}\\
\cmidrule(lr){2-3}
\cmidrule(lr){4-5}
\cmidrule(lr){6-7}
\cmidrule(lr){8-9}
\cmidrule(lr){10-11}
{\textbf{Method}} & {PSNR $\uparrow$} & {SSIM $\uparrow$} & {PSNR $\uparrow$} & {SSIM $\uparrow$} & {PSNR $\uparrow$} & {SSIM $\uparrow$} & {PSNR $\uparrow$} & {SSIM $\uparrow$} & {PSNR $\uparrow$} & {SSIM $\uparrow$}\\
\midrule

ADMM-TV & 22.17 & 0.679 & 17.96 & 0.785 & 20.96 & 0.676 & 19.99 & 0.634 & 20.79 & \underline{0.677} \\
Score-SDE~\citep{song_2020_score} & 12.25 & 0.256 & 16.48 & 0.612 & 18.62 & 0.517 & 15.97 & 0.436 & 7.21 & 0.120 \\
PnP-ADMM~\citep{chan2016plug} & 23.75 & 0.761 & 12.70 & 0.657 & 8.39 & 0.300 & 21.81 & 0.669 & 21.98 & \textbf{0.702} \\
MCG~\citep{chung2022improving} & 13.39 & 0.227 & 17.36 & 0.633 & 19.03 & 0.546 & 16.32 & 0.441 & 5.89 & 0.037 \\
DDRM~\citep{kawar2022denoising} & \underline{24.96} & \textbf{0.790} & 18.66 & \textbf{0.814} & 14.29 & 0.403 & \underline{22.73} & 0.705 & - & -\\
DPS~\citep{chung2022diffusion} & 23.87 & \underline{0.781} & \underline{18.90} & 0.794 & \underline{22.20} & \underline{0.739} & 21.97 & \underline{0.706} & \underline{20.55} & 0.634\\
\textbf{ADMMDiff} & \textbf{25.04} & 0.688 & \textbf{19.03} & \textbf{0.814} & \textbf{27.93} & \textbf{0.824} & \textbf{23.32} & \textbf{0.716} & \textbf{25.11} & 0.670 \\

\bottomrule
\end{tabular}
}
\caption{
Quantitative evaluation (PSNR, SSIM) of solving linear inverse problems on ImageNet 256$\times$256-1k validation dataset. \textbf{Bold} indicates the best. \underline{Underline} indicates the second best.
}
\label{tab:nip_imagenet_psnr_ssim}
\end{table*}

In Figure~\ref{fig:motion}, we compare the visual results of our method with GMD without trajectory inpainting.

\begin{figure*}[t]
    \centering
    \includegraphics[width=0.99\linewidth]{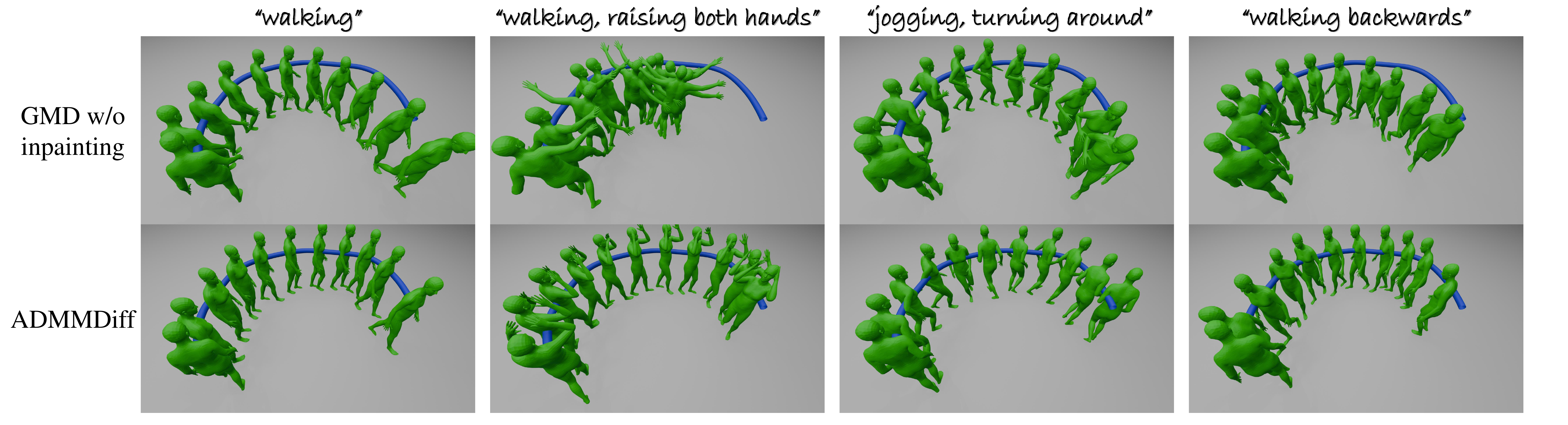}
    \caption{Qualitative comparison on controllable motion generation using trajectory guidance. Blue line indicates the path to follow. The results show that our method better follows the trajectory while being consistent with the motion prior and text prompt.}
    \label{fig:motion}
\end{figure*}

\end{document}